\documentclass[draftcls,onecolumn,11pt]{IEEEtran}
\usepackage[dvips]{graphicx}
\usepackage{amsmath,amstext,amssymb,epsf}
\usepackage{fullpage,latexsym}
\usepackage{epsfig}
\usepackage{setspace}

\numberwithin{equation}{section} 

\newtheorem{theorem}{\sc Theorem}

\newtheorem{lemma}{\sc Lemma}
\newtheorem{coro}{\sc Corollary}
\newtheorem{req}{\sc Requirement}
\newtheorem{nota}{\sc Notation}
\newtheorem{defin}{\sc Definition}
\newtheorem{rem}{\sc Remark}
\newtheorem{cla}{\sc Claim}
\newtheorem{ex}{\sc Example}

\newenvironment{proof}{\par \sc Proof.\rm}{\hspace*{\fill}$\bullet$\vspace{1ex}}
\newenvironment{example}{\begin{ex}}{\hspace*{\fill}$\diamondsuit$\end{ex}}
\newenvironment{claim}{\begin{cla}}{\end{cla}}
\newenvironment{corollary}{\begin{coro}}{\end{coro}}

\newenvironment{definition}{\begin{defin}}{\end{defin}}
\newenvironment{remark}{\begin{rem}}{\hspace*{\fill}$\Diamond$\end{rem}}

\renewcommand{\emptyset}{\varnothing}

\begin{document}

\title{Identification of Probabilities}
\author{Paul M.B. Vit\'anyi and Nick Chater
\thanks{Vit\'anyi is with the 
National Research Institute for Mathematics
and Computer Science in the Netherlands (CWI) and
the University of Amsterdam.
Address: CWI, Science Park 123, 
1098 XG, Amsterdam, The Netherlands.
Email: {\tt paulv@cwi.nl}. 

Chater is with the 
Behavioural Science Group.
Address: Warwick Business School, University of Warwick, Coventry, CV4 7AL, UK.
Email: {\tt Nick.Chater@wbs.ac.uk}. Chater was supported by ERC grant 295917-RATIONALITY, the ESRC Network for Integrated Behavioural Science [grant number ES/K002201/1], the Leverhulme Trust [grant number RP2012-V-022], and Research Councils UK Grant EP/K039830/1.
}}

\maketitle
\begin{abstract}
Within psychology, neuroscience and artificial intelligence, there has been increasing interest in the proposal that the brain builds probabilistic models of sensory and linguistic input: that is, to infer a probabilistic model from a sample. 
The practical problems of such inference are substantial: the brain has limited data and restricted computational resources. But there is a more fundamental question: is the problem of inferring a probabilistic model from a sample possible even in principle? We explore this question and find some surprisingly positive and general results. First, for a broad class of probability distributions characterised by computability restrictions, we specify a learning algorithm that will almost surely identify a probability distribution in the limit given a finite i.i.d. sample of sufficient 
but unknown length. 
This is similarly shown to hold for sequences generated by a broad class 
of Markov chains, subject to computability assumptions.
The technical tool is the strong law of large numbers. 
Second, for a large class of dependent sequences, we specify an algorithm which identifies in the limit
a computable measure for which the sequence is typical, in the sense of Martin-L\"of (there 
may be more than one such measure). 
The technical tool is the theory of Kolmogorov complexity.
We analyse the associated predictions in both cases. We also briefly consider special cases, including language learning, and wider theoretical implications for psychology. 

Keywords: learning, Bayesian brain, identification, 
computable probability, Markov chain, computable measure,
typicality, strong law of large numbers, Martin-L\"of randomness, 
Kolmogorov complexity
\end{abstract}
\date{}
\section{Introduction}\label{sect.0}
Bayesian models in psychology and neuroscience postulate that the brain learns a generative probabilistic model of a set of perceptual or linguistic data (\cite{CTY06}, \cite{OC07}, \cite{PBML13}, \cite{TKGG11}). Learning is therefore often viewed as an 'inverse'-problem. Some aspect of the world is presumed to contain a probabilistic model, from which data is sampled; the brain receives a sample of such data, e.g., at its sensory surfaces, and has the task of inferring the probabilistic model. For example, the brain has to infer an underlying probability distribution, from a sample from that distribution. 

This theoretical viewpoint is implicit in a wide range of Bayesian models in cognitive science, which capture experimental data across many domains, from perception, to categorization, language, motor control, and reasoning (e.g., \cite{CO08}). It is, moreover, embodied in a wide range of computational models of unsupervised learning in machine learning, computational linguistics, computer vision (e.g., \cite{Ack85}, \cite{MK03}, \cite{YK06}). Finally, the view that the brain recovers probabilistic models from sensory data is both theoretically prevalent and has received considerable empirical support in neuroscience (\cite{KP04}). 

The idea that the brain may be able to recover a probabilistic process from a sample of data from that process is an attractive one. For example, a recovered probabilistic model might potentially be used to explain past input or to predict new input. Moreover, sampling new data from the recovered probabilistic model could be used in the generation of new data from that probabilistic process, from creating mental images \cite{Sh84} or producing language \cite{CV07}. Thus, from a Bayesian standpoint, one should expect that the ability to perceive should go alongside the ability to create 'mental images;' and the ability to understand language should go alongside the ability to produce language. Thus, the Bayesian approach is part of the broader psychological tradition of analysis-by-synthesis, for which there is considerable behavioural and neuroscientific evidence with a large amount of evidence, in perceptual and linguistic domains (\cite{PG13}, \cite{YK06}).

Yet, despite its many attractions, the proposal that the brain recovers probabilistic processes from samples of data faces both practical and theoretical challenges. The practical challenges include the fact that the available data may be limited (e.g., children learn the probabilistic model of highly complex language using only millions of words). Moreover, the brain faces severe computational constraints: even the limited amount of data encountered will be encoded imperfectly and may rapidly be lost (\cite{CC15}, \cite{Ha83}). The brain has limited processing resources to search and test the vast space of possible probabilistic models that might generate the data available.

In this paper we explore the conditions under which exactly inferring a probabilistic process from a stream of data is possible even in principle, with no restrictions on computational resources like time or storage or availability of data. If it turns out that there is no algorithm that can learn a probabilistic structure from sensory or linguistic experience when no computational or data restrictions are imposed, then this negative result will still hold when more realistic settings are examined. 

Our analysis differs from previous approaches to these issues by assuming that the probabilistic process to be inferred is, in a way that will be made precise later, computable. Roughly speaking, the assumption is that the data to be analysed is generated by a process that can be modelled by a computer (e.g., a Turing machine or a conventional digital computer) combined with a source of randomness (for example, a fair coin that can generate a limitless stream of random 0s and 1s that could be fed into the computer). There are three reasons to suppose that this focus on computable processes is interesting and not overly restrictive. First, some influential theorists have argued that all physical processes are computable in this, or stricter, senses (e.g., \cite{De85}). Second, most cognitive scientists assume that the brain is restricted to computable processes, and hence can only {\it represent} computable processes (e.g., \cite{Re15}). According to this assumption, if it turns out that some aspects of the physical world are uncomputable, these will trivially be unlearnable simply because they cannot be represented; and, conversely, all aspects of learning of relevance to psychology, i.e., all aspects of the world that the brain can successfully learn, will be within the scope of our analysis. Third, all existing models of learning in psychology, statistics and machine learning are computable (and, indeed, are actually implemented on digital computers) and fall within the scope of the present results.

\subsection{Background: Pessimism about learnability}

Within philosophy of science, cognitive science, and formal learning theory, a variety of considerations appear to suggest that negative results are likely. For example, in the philosophy of science it is often observed that theory is underdetermined by data (\cite{Du14}, \cite{Qu51}): that is, in infinite number of theories is compatible with any finite amount of data, however large. After all, these theories can all agree on any finite data set, but diverge concerning any of the infinitely large set of possible data that has yet to be encountered. This might appear to rule out identifying the correct theory---and hence, {\it a fortiori} identify a correct probability distribution. 

Cognitive science inherits such considerations, to the extent that the learning problems faced by the brain are analogous to those of inferring scientific theories (e.g., \cite{GMK99}). But cognitive scientists have also amplified these concerns, particularly in the context of language acquisition. Consider, for example, the problem of acquiring language from positive evidence alone, i.e., from hearing sentences of the language, but with no feedback concerning whether the learner's own utterances are grammatical or not (so-called negative evidence). It is often assumed that this is, to a good approximation, the situation by the child. This is because some and perhaps all children receive little useful feedback on their own utterances and ignore such feedback even when it is given (\cite{Bo88}). Yet, even without negative evidence, children nonetheless learn their native language successfully. 
For example, an important textbook on language acquisition \cite{CL99}
 repeatedly emphasises that the child cannot learn {\it restrictions} on grammatical rules from experience---and that these must therefore somehow arise from innate constraints. For example, the English sentences {\it which team do you want to beat}, {\it which team do you wanna beat}, and {\it which team do you want to win}, which would seem naturally to imply that {\it *which team do you wanna win} is also a grammatical sentence. As indicated by the asterisk, however, this sentence is typically rejected as ungrammatical by native speakers. According to classical linguistic theory (e.g., \cite{Ch82}), the contraction to {\it wanna} is not possible because it is blocked by a ``gap'' indicating a missing subject---a constraint that has sometimes been presumed to follow from an innate universal grammar \cite{Ch80}. 

The problem with learning purely from positive evidence is that an overgeneral hypothesis, which does not include such restrictions, will be consistent with new data; given that languages are shot through with exceptions and restrictions of all kinds, this appears to provide a powerful motivation for a linguistic nativism \cite{Ch80}. But this line of argument cannot be quite right, because many exceptions are entirely capricious and could not possibly follow from innate linguistic principles. For example, the grammatical acceptability of {\it I like singing}, {\it I like to sing}, and {\it I enjoy singing} would seem to imply, wrongly, the acceptability {\it *I enjoy to sing}. But the difference between the distributional behaviour of the verbs {\it like} and {\it enjoy} cannot stem from any innate grammatical principles. The fact that children are able to learn restrictions of this type, and the fact that they are so ubiquitous throughout language, has even led some scholars to speak of the {\it logical} problem of language acquisition (\cite{BM81}, \cite{HL81}).

Similarly, in learning the meaning of words, it is not clear how, without negative evidence, the child can successfully retreat for overgeneralization. If the child initially proposes that, for example, {\it dog} refers to any animal, or that {\it mummy} refers to any adult female, then further examples will not falsify this conjecture. In word learning and categorization, and in language acquisition, researchers have suggested that one potential justification for over-turning an overgeneral hypothesis is that absence-of-evidence can sometimes be evidence-of-absence (\cite{Hsu16, HO08}). That is, a child might take the absence of people using the word {\it dog} when referring to cats or mice; and the absence of {\it Mummy} being used to refer to other female friends or family members might lead to the child to be in doubt concerning their liberal use of these terms. But, of course, this line of reasoning is not straightforward---for example, when learning {\it any} category that may apply in an infinite number of situations, the overwhelming majority of these will not have been encountered. It is not immediately clear how the child can tell the difference between benign, and genuinely suspicious, absence of evidence. The present results show that there is an algorithm that, under fairly broad conditions, can deal successfully with overgeneralization with probability 1, given sufficient data and computation time. 

Previous results in the formal analysis of language learnability have reached more pessimistic conclusions, using different assumptions (\cite{Go67}, \cite{JORS99}). For example, as quoted in \cite{Pi79}, the pioneer of formal learning theory E. M. Gold points that ``the problem with [learning only from] text is that if you guess too large a language, the sample will never tell you you're wrong'' (\cite{Go67}, p. 461). This is true if we allow very few assumptions about the structure of the text---and indeed negative results in this area frequently depend on demonstrating the existence of texts (i.e., samples of the language) with rather unnatural behavior precisely designed to mislead any putative learner. We shall see below that realistic, though still quite mild, assumptions, are sufficient to yield the opposite conclusion: that probability distributions, including probability distributions over languages, can be identified from positive instances alone. 


\subsection{Preview and examples}

Consider, first, the case of independent, identical draws from a probability distribution. In many areas of psychology, the learning task is viewed as abstracting some pattern from a series of independent trials rather than picking up sequential regularities (although the i.i.d. assumption is not necessarily explicit). The i.i.d. case is relevant to problems as diverse as classical conditioning (\cite{RW}, where a joint distribution between conditioned and unconditioned stimuli must be acquired) category learning (\cite{SHJ}, where a joint distribution of category instances and labels is the target), artificial grammar learning or artificial language learning (\cite{Reber}, \cite{Saffran}, where a probability distribution over strings of letters or sounds is to be learned). Similarly, the i.i.d. assumption is often implicit in learning algorithms in cognitive science and machine learning, such as, for example, many Bayesian and neural network models in perception, learning and categorization (e.g., \cite{Ack85})

Learning such potentially complex patterns from examples may seem challenging. Yet even analysing perhaps the simplest case, learning the probability distribution of a biased coin is not straightforward. 
For concreteness, consider flipping a coin, with probability $p$ of coming up heads. Suppose that we can flip the coin endlessly, and can, at every point as the sequence of data emerges, guess the value of $p$; we can change our mind as often as we like. It is natural to wonder whether there is some procedure for guessing such that, after some point, we stick to our guess---and that this guess is, either certainly or with high probability, correct. So, for example, if the coin is a fair coin, such that $p=0.5$, will we eventually lock on to the conjecture that the coin is fair and, after some point, never change this conjecture however much data we receive?

The answer is by no means obvious, even for such simple case. After all, the difference between the number of heads and tails will fluctuate, and can grow arbitrarily large---and such fluctuations might persuade us, wrongly, that the coin is biased in favour, or against, heads. How sure can we be that, eventually, we will successfully identify the precise bias of a coin that {\it is} biased, e.g., where $p=3/4$ or $p=1/3$?

Or, to step up the level of complexity very considerable, consider the problem of inferring a stochastic phrase structure grammar from an indefinitely large sample of i.i.d. sentences generated from that grammar. \footnote{A stochastic phrase structure grammar is a conventional phrase structure grammar, with probabilities associated with each of the rewrite rules. For example, a noun phrase might sometimes 'expand' to give a determiner followed by a noun, while sometimes expanding to give a single 'proper' noun; and individual grammatical categories, such as proper nouns, map probabilistically on specific proper nouns.} Or suppose the input is a sequence of images generated draw from a probabilistic image model such as a Markov random field---can a perceiver learn to precisely identify the probabilistic model of the image, given sufficient data? 

As we shall see in Section III, below, remarkably, it turns out that, with fairly mild restrictions (a restricted computability), with probability 1, it is possible to infer in the limit, the correct probability distribution exactly, given a  sufficiently large finite supply of i.i.d. samples. Moreover, it is possible to specify a computable algorithm that will reliably find this probability distribution. A similar result holds for ergodic Markov chains, which broadens its application considerably.


This result is unexpectedly strong, given mild restrictions on computability (which we describe in detail below). In particular, it shows that there is no {\it logical} problem concerning the possibility of learning languages, or other patterns, which contain exceptions, from positive evidence alone. As noted above, it has been influentially argued in linguistics and the study of language acquisition that exceptions (examples that are {\it not} possible) cannot be learned purely by observing their non-occurrence, because there are, after all, infinitely many linguistic forms which are possible but also have not been observed (e.g., \cite{CL99}). A variety of arguments and results have suggested that, despite such arguments, languages with exceptions can be learned successfully (\cite{CCGP15}, \cite{CL10}, \cite{PS02}). 


The present result shows that with the mentioned restrictions, given i.i.d. data it is possible exactly to learn the probability distribution
of languages from a sample; or, from Markovian outputs, it is possible exactly to learn the Markov chain involved, and hence, of course, which sentences are acceptable in the language. An earlier result in \cite{CV07} showed that language acquisition with sufficient data was possible on the assumption that an 'ideal' learner could find the shortest description of a corpus. But finding the shortest description is known to be incomputable. By contrast, the present paper focuses on what can be learned by a computable learner, provides an explicit algorithm by which that learner can operate, and considers exact learning rather than approximating the language arbitrarily accurately. 

We also consider what can be learned if we weaken the i.i.d. restriction (and the mentioned Markov chain restriction) considerably---to deal with the possibility of learning sequential data that is generated by a computable process (we make this precise below). Many aspects of the environment, from the flow of visual and auditory input, to the many layers of sequential structure relating successive sentences, paragraphs, and chapters, while reading a novel, are not well approximated by identical independent sampling from a fixed distribution or the output of a small Markov chain. Nonetheless, the brain appears to be able to discover their structure, at least to some extent, with remarkable effectiveness. 


One particularly striking illustration of the power to predict subsequent input is Shannon''s method for estimating the entropy of English \cite{Sh51}. Successively predicting the next letter in a text, given previous letters, one or two guesses often suffice, leading to the conclusion that English texts typically can be encoded using little more than one bit of information per letter (while more than four bits would be required if the 26 letters were treated as occurring independently). The ability to predict incoming sequential input is, of course, important for reacting to the physical or linguistic environment successfully, by predicting dangers and opportunities and acting accordingly. Many theorists also see finding structure in sequential material as fundamental to cognition and learning (\cite{Cl15}, \cite{El90}, \cite{HS98}, \cite{KFF07}).

If we weaken the i.i.d. or above Markovian assumption, what alternative restriction on sequential structure can we impose, and still obtain tractable analytical results? Clearly if there are no restrictions on structure of the process at all, then there are no constraints between prior and subsequent material. It turns out, though, a surprisingly minimal restriction is sufficient: we assume, roughly, only that the sequential material is generated by a mildly restricted {\em computable} dependent probabilistic process (this will be made precise below). Unlike the i.i.d. or Markov case, different such processes could have generated this sample; but it turns out that, given a finite sample that is long enough and that is guaranteed to be the initial segment of an infinite typical output of one of those computable dependent probabilistic processes, it is possible to infer a single process exactly (out of a number of such processes) according to which that sample is an initial segment of an infinite typical sample. We shall discuss these issues in Section IV.

Throughout this paper, we focus on learning probabilities themselves, rather than particular representations of probabilities. If there is at least one computer program representing a function, there are, of course, infinitely many such programs (representing the data in slightly different ways, incorporating additional null operations, and so on). The same is true for programs representing probability distributions. For some purposes, these differences in representation may be crucial. For example, psychologists and linguists may be interested in which of an infinite number of equivalent grammars---from the point of view of the sentences allowed---is represented by the brain. But, from the point of view of the problem of learning, we must treat them as equivalent. Indeed, it is clear that no learning method of probabilities alone could ever distinguish between models which generate precisely the same probability distribution over possible observations. 

Our discussion begins with an introduction of our formal framework, in the next section. We then turn to the case of i.i.d. draws from a computable mass function, and to runs of a computable ergodic Markov chain, using the strong law of large numbers as the main technical tool. The next section {\it Computable Measures} considers learning material with computable sequential dependencies; here the main technical tool is Kolmogorov complexity theory. We then briefly consider whether these results have implications for the problem of prediction future data, based on past data, before we draw brief conclusions. The mathematical details and detailed proofs are relegated to Appendices. 

\section{The formal framework}
We follow in the general theoretical tradition of formal learning theory, where we abstract away from specific representational questions, and focus on the underlying abstract structure of the learning problem. 

One can associate the natural numbers 
with a lexicographic length-increasing ordering of finite strings
over a finite alphabet. A natural number corresponds to the string
of which it is the position in the thus established order. Since
a language is a set of sentences (finite strings over a finite alphabet),
it can be viewed as a subset of the natural numbers. (In the same way, natural numbers could be associated with images or instances of a concept). 
The learnability of a language under various
computational assumptions is the subject of an immensely influential
approach in \cite{Go65} and especially \cite{Go67}, or the review 
\cite{JORS99}. But surely
in the real world the chance of one sentence of a language 
being used is different
from another one. For example, in general short sentences have a larger
chance of turning up than very long sentences. Thus, the elements
of a given language are distributed in a certain way. There arises
the problem of identifying or approximating this distribution.    

Our model is formulated as follows: 
we are given an sufficiently long finite sequence of data consisting of elements
drawn from the set (language) according to a certain
probability, and the learner has to identify this probability. 
In general, however much data been encountered, 
there is no point at which the learner can
announce a particular probability as correct with certainty.
Weakening the learning model, the learner might learn to
identify the correct probability in the limit. That is, perhaps the learner
might make a sequence of guesses, finally locking 
on to correct probability and sticking to it
forever---even though the learner can never know for sure that it has
identified the correct probability successfully. 
We shall consider
identification in the limit (following, for example, 
\cite{Go67,JORS99,Pi79}). 
Since this is not enough we additionally restrict
the type of probability.

In conventional statistics, probabilistic models are typically idealized
as having continuous valued parameters; and hence there is an uncountable
number of possible probabilities.
In general it is impossible that a learner 
can make a sequence of guesses
that precisely locks on to the correct values of continuous parameters.
In the realm of algorithmic information theory, in particular in
Solomonoff induction \cite{So64} and here, we reason as follows.
The possible strategies of learners are computable in the sense
of Turing \cite{Tu36}, that is, they are computable functions.
The set of these is discrete and thus countable.
The hypotheses that can be learned are therefore countable, 
and in particular
the set of probabilities from which the learner chooses 
must be {\em computable}. Indeed, this argument can be interpreted as showing that the fundamental problem is one of representation: the overwhelming majority of real-valued parameters cannot be {\it represented} by any computable strategy; and hence {\it a fortiori} cannot possible be learned. 

Our starting point is that it is only of interest to consider the identifiability of computable hypotheses---because hypotheses that are not computable cannot be represented, let alone learned. Making this precise requires specifying what it means for a probability distribution to be computable. Moreover, it turns out
that computability is not enough, it is also necessary that the considered set of computable probabilities is computably enumerable (c.e.) and co-computable enumerable (co-c.e.) sets, all of which are explained in the Appendix~\ref{sect.comput}.  Informally, a subset of a set is c.e. 
if there is a computer which 
enumerates all the elements of the subset but no element
outside the subset (but in the set). 
For example, the computable probability
mass functions (or computable measures) for which we know algorithms
computing them can be computably enumerated in lexicographic
order of the algorithms. Hence they
satisfy Theorem~\ref{theo.1} (or Theorem~\ref{theo.3}.

A subset is co-c.e. if all elements outside the subset (but in the set) 
can be enumerated by a computer. In our case the set comprises 
all computable probability mass functions, respectively, 
all computable measures. 
Since by Lemma~\ref{lem.incomp} in Appendix~\ref{sect.comput} this set is not
c.e. a subset that is c.e. (or co-c.e.) is a proper subset, that is, it
does not contain all computable probability mass functions, respectively,
all computable measures.


In the exposition below, we consider two cases. 
In case 1 the data are
drawn independent identically distributed (i.i.d.) from a 
subset of the natural numbers according to a probability mass function
in a c.e. or co-c.e. set of computable probability mass functions, or
consist of a run of a member of
a c.e. or co-c.e. set of computable ergodic Markov chains. For this case, there is, as we have noted, a learning algorithm that will almost surely identify a probability distribution in the limit. This is the topic of Section III, below. 

In case 2 the elements of the infinite sequence 
are dependent 
and the data sequence is typical for a measure
from a c.e. or co-c.e. set of computable measures. For this more general  case, we prove a weaker, though still surprising result: that there is an algorithm which identifies in the limit a computable measure for which that sequence is typical (in the sense introduced by Martin-L\"of). These results are the focus of Section IV, below.

\subsection{Preliminaries}
Let ${\cal N}$, ${\cal Q}$, ${\cal R}$, and ${\cal R}^+$ 
denote the natural numbers, 
the rational numbers, the real numbers, and the nonnegative
real numbers, respectively.
We say that we {\em identify} a function
$f$ {\em in the limit} if we have an algorithm which
produces an infinite sequence $f_1, f_2, \ldots$ of functions and 
$f_i=f$ for all but finitely many $i$. This corresponds to
the notion of ``identification in the limit'' in 
\cite{Go67,JORS99,Pi79,ZZ08}.
In this notion at every step
an object is produced and after a 
finite number of steps the target object is
produced at every step. However, we do not know this finite number. It is as
if you ask directions and the answer is ``at the last intersection 
turn right,'' but you do not know which intersection is last.
The functions $f$ we want to identify in the limit are
probability mass functions, Markov chains, or measures.
\begin{definition}
Let $L \subseteq {\cal N}$ and its associated 
{\em probability mass function} $p$ a function
$p: L \rightarrow {\cal R}^+$ satisfying 
$\sum_{x \in L} p(x)=1$. A Markov chain is an extension
as in Definition~\ref{def.m}. A {\em measure} $\mu$
is a function $\mu : L^* \rightarrow {\cal R}^+$ satisfying
the measure equalities in Appendix~\ref{sect.measure}.
\end{definition}

\subsection{Related work}\label{sect.rel}
In \cite{An88} (citing previous more restricted work)
a target probability mass function was identified in the limit
when the data are drawn i.i.d. in the following setting.
Let the target probability mass function $p$ be an element
of a list  $q_1, q_2, \ldots $ subject to the following
conditions: 
(i) every $q_i : {\cal N} \rightarrow {\cal R}^+$
is a probability mass function; 
(ii) we exhibit a computable total function $C(i,x,\epsilon)=r$ such that
$q_i(x)-r \leq \epsilon$ with $r,\epsilon >0$ are rational numbers.
That is, there exists a rational number approximation for all probability
mass functions in the list up to arbitrary precision, and we give
a single algorithm which for each such
function exhibits such an approximation. 
The technical means used
are the law of the iterated logarithm and the Kolmogorov-Smirnov test. 
However, the list $q_1, q_2, \ldots $ can not contain 
all computable probability
mass functions because of a diagonal argument, Lemma~\ref{lem.incomp}.

In \cite{BC91} computability questions are apparently ignored.
The {\em Conclusion} 
states ``If the true density [and hence
a probability mass function] is finitely complex [it is computable]
then it is exactly discovered for all sufficiently large sample sizes.''.
The tool that is used is estimation according to 
$\min_q (L(q)+\log(1/\prod_{i=1}^n q(X_i))$. Here $q$ is a probability mass 
function, $L(q)$ is the length of its code and $q(X_i)$ is
the $q$-probability of the $i$th random variable $X_i$. To be able to
minimize over the set of computable $q$'s,
one has to know the $L(q)$'s. 
If the set of candidate distributions is countably infinite, then
we can never know when the minimum is reached---hence at best we have
then identification in the limit.
If $L(q)$ is identified with the Kolmogorov complexity $K(q)$, 
as in Section IV of this reference,
then it is incomputable as already
observed by Kolmogorov in \cite{Ko65} 
(for the plain Kolmogorov complexity; the case of
the prefix Kolmogorov complexity $K(q)$ is the same). 
Computable $L(q)$
(given $q$) cannot be computably enumerated; if they were this
would constitute a computable enumeration of computable $q$'s which 
is impossible by Lemma~\ref{lem.incomp}. To obtain the minimum we
require a computable enumeration of the $L(q)$'s in the estimation formula.
The results hold (contrary to what is claimed in the {\em Conclusion}
of \cite{BC91} and
other parts of the text) not for the set of computable 
probability mass functions since they are not c.e..
The sentence
``you know but you don't know you know'' on the second page of \cite{BC91}
does not hold for an arbitrary
computable mass probability.

In reaction to an earlier version of this paper with too
large claims as described in Appendix~\ref{sect.genesis}, in \cite{BMS14} it is shown that it is 
impossible to identify an arbitrary
computable probability mass function (or measure) 
in the limit given an infinite 
sequence of elements from its support (which sequence is guarantied
to be typical for some computable measure in the measure case). 

\subsection{Results}
The set of halting algorithms for computable probabilities (or measures) is 
not c.e., 
Lemma~\ref{lem.incomp} in Appendix~\ref{sect.comput}. 
This complicates the algorithms and
analysis of the results. 
In Section~\ref{sect.1} there is
a computable probability mass function (the target) 
on a set of natural numbers.
We are given a sufficiently long finite sequence of
elements of this set that are drawn i.i.d. and are 
asked to identify the target.
An algorithm is presented which
identifies the target in the limit almost surely
provided the target is an element of a c.e. or co-c.e. 
set of halting algorithms
for computable probability mass functions 
(Theorem~\ref{theo.1}). This also
underpins the result announced in \cite[Theorem 1 in the Appendix and 
appeals to it in the main text of the reference]{HCV11} with the following
modification ``computable probabilities''
need to be replaced by ``c.e. and co-c.e. sets of computable probabilities''.
If the target is an element of a c.e. or co-c.e.
set of computable ergodic Markov chains
then there is an algorithm with as input a sequence of states 
of a run of the Markov chain and as output almost surely the target
(Corollary~\ref{cor.2}). 
The technical tool is in both cases the strong law of large numbers.
In Section~\ref{sect.3} the set of natural numbers 
is also infinite and the elements 
of the sequence are allowed 
to be dependent. We are given a guaranty that the sequence is 
typical (Definition~\ref{def.typical})
for at least one measure from a c.e. or co-c.e. set 
of halting algorithms for computable measures.
There is an algorithm which identifies in the limit 
a computable measure for which the data sequence is typical
(Theorem~\ref{theo.3}).
The technical tool is the Martin-L\"of theory of sequential
tests \cite{Ma66} based on Kolmogorov complexity.
In Section~\ref{sect.4} we consider the associated predictions, and in
Section~\ref{sect.concl} we give conclusions.
In Appendix~\ref{sect.comput} we review the used computability notions,
in Appendix~\ref{sect.kolmcomp} we review notions of Kolmogorov complexity,
in Appendix~\ref{sect.measure} we review the used measure and computability
notions. We defer the proofs of the theorems to Appendix~\ref{sect.proofs}.
In Appendix~\ref{sect.genesis} we give the tortuous genesis of the results.

\section{Computable Probability Mass Functions and I.I.D. Drawing}\label{sect.1}
To approximate a probability in the i.i.d. setting is 
well-known and an easy 
example to illustrate our problem. 
One does this by an algorithm computing the probability
$p(a)$ in the limit for all $a \in L \subseteq {\cal N}$ 
almost surely given the infinite sequence $x_1 ,x_2, \ldots$
of data i.i.d. drawn from $L$ according to $p$. 
Namely, for $n=1,2, \ldots$ for every $a \in L$ occurring 
in $x_1 ,x_2, \ldots , x_n$
set $p_n(a)$ equal to the frequency of occurrences of
$a$ in $x_1,x_2, \ldots ,x_n$.
Note that the different values
of $p_n$ sum to precisely 1 for every $n=1,2, \ldots .$
The output is a sequence $p_1, p_2, \ldots$ of 
probability mass functions such that we have
$\lim_{n \rightarrow \infty} p_n=p$ almost surely, 
by the strong law of large numbers (see Claim~\ref{claim.slln}).
The probability mass functions considered here consist of {\em all}
probability mass functions on $L$---computable or not. 
The probability mass function $p$ is 
thus represented by an approximation algorithm.

In this paper we deal only with computable probability mass functions.
If $p$ is computable then
it can be represented by a halting algorithm which computes it as 
defined in Appendix~\ref{sect.comput}.
Most known probability mass functions are computable provided
their parameters are computable. 
In order that it is computable
we only require that the probability mass function is finitely
describable and there is a computable process producing it \cite{Tu36}.

One issue is how short the code for $p$ is, a second issue
are the computability properties of the code for $p$, a third
issue is how much of the data sequence is used in the learning
process. The approximation 
of $p$ above results in a sequence of codes 
of probabilities $p_1,p_2, \ldots$
which are lists of the sample frequencies in an initial finite
segment of the data sequence. The code length of the list of frequencies 
representing $p_n$ grows usually
to infinity as the length $n$ of the segment grows to infinity. 
The learning process involved
uses all of the data sequence and the result is an encoding of
the sample frequencies in the data sequence in the limit. 
The code for $p$ is usually infinite.
This holds as well if $p$ is computable. Such an approximation contrasts
with identification in the following. 


\begin{theorem}\label{theo.1}
{\sc I.I.D. Computable Probability Identification}
Let $L$ be a set of natural numbers and
$p$ be a probability mass function on $L$. This $p$ is described by
an element of a c.e. or co-c.e. set of halting algorithms
for computable probability mass functions. 
There is an algorithm identifying $p$ in the limit almost surely from
an infinite sequence $x_1,x_2, \ldots$ of elements of $L$ drawn i.i.d.
according to $p$. 
The code for $p$ via an appropriate Turing machine
is finite. The learning process uses only a finite initial segment
of the data sequence and takes finite time. 
\end{theorem}
We do not know how large the finite items in the theorem are.
The proof of the theorem is 
deferred to Appendix~\ref{sect.proofs}. The intuition is as follows.
By assumption the target probability mass function 
is a member of a  linear list of halting algorithms for 
computable probability mass functions listed as list ${\cal A}$. 
By the strong law of large numbers we can approximate the target
probability mass function by the sample means. Since the members
of ${\cal A}$ are linearly ordered we can after each new sample
compute the least member which agrees best according to a certain
criterion with the samples
produced  thus far. At some stage this least element
does not change any more. 
\begin{example}
\rm
Since the c.e. and co-c.e. sets strictly contain 
the computable sets, Theorem~\ref{theo.1}
is strictly stronger than the result in \cite{An88} referred to in
Section~\ref{sect.rel}. It is also strictly stronger than 
\cite{BC91} that does not give identification in the limit
for classes of computable functions. 

Define the primitive computable probability mass functions 
as the set of probability mass functions
for which it is decidable that they are constructed from
primitive computable functions. Since this set is computable it is c.e..
The theorem shows that
identification in the limit is possible for members of this set.
Define the time-bounded probability
mass functions for any fixed computable time bound as
the set of elements for which it is decidable 
that they are computable probability
mass functions
satisfying this time bound. Since this set is computable it is c.e..
Again, the theorem shows that
identification in the limit is possible for elements from this set.

Another example is as follows.
Let $L=\{a_1,a_2, \ldots , a_n\}$ be a finite set. The primitive recursive
functions $f_1,f_2, \ldots$ are c.e.. Hence the probability mass functions
$p_1, p_2, \ldots$ on $L$ defined by $p_i(a_j)= f_i(j)/\sum_{h=1}^n f_i(h)$
are also c.e.. Let us call these probability mass functions
simple. By  Theorem~\ref{theo.1} they can be identified
in the limit.
\end{example}

The class of probability mass functions for which the present result applies 
is very broad. Suppose, for example, that we frame the problem of language 
acquisition in the following terms: a corpus is created by i.i.d. sampling 
from some primitive recursive language generation mechanism 
(for example, a stochastic phrase structure grammar \cite{Ch96} with 
rational probabilities, or an equivalent, but more cognitively motivated 
formalism such as tree-adjoining grammar \cite{JS97} or combinatory 
categorical grammar \cite{St00}). That is, the algorithm described here will search possible programs which correspond to generators of grammars, and will eventually find, and never change from, a stochastic grammar that precisely captures the probability mass function that generated the linguistic data. 
That is, the present result implies that there is a learning algorithm that 
identifies in the limit the probability mass function according to which 
these sentences are generated with probability 1. Of course, there may, in general, within any reasonably rich stochastic grammar formalism, be many ways of representing the probability distribution over possible sentences (just as there are many computer programs that code for the same function). Of course, no learning process can distinguish between these, precisely because they are, by assumption, precisely equivalent in their  predictions. Hence, an appropriate goal of learning can only be to find the underlying probability mass function, rather than attempting the impossible task of inferring the particular representation of that function. 

The result applies, of course, not just to language but to learning structure in perceptual input, such as visual images. Suppose that a set of visual images is created by i.i.d sampling 
from a Markov random field with rational parameters \cite{Li12}; 
then there will be a learning algorithm which identifies in the limit 
the probability distribution over these images with probability 1. The result applies, also, to the unsupervised learning of environmental structure from data, for example by connectionist learning methods \cite{Ack85} or by Bayesian learning methods (\cite {CTY06}, \cite{Pe88}, \cite{TKGG11}).

\subsection{Markov chains}
I.i.d. draws from a probability mass function is a special case of
a run of a discrete Markov chain. 
We investigate which Markov chains have an equivalent
of the strong law of large numbers.
Theorem~\ref{theo.1} then holds {\em mutatis mutandis} 
for these Markov chains.
First we need a few definitions.
\begin{definition}\label{def.m}
\rm
A sequence of random variables $(X_t)_{t=0}^{\infty}$ with
outcomes in a finite or countable state space $S \subseteq {\cal N}$
is a {\em discrete time-homogeneous Markov chain}
if for every ordered pair $i,j$ of states the quantity
$q_{i,j} = \Pr (X_{t+1} = j | X_t = i)$ called the
{\em transition probability} from state $i$ to state $j$,
is independent of $t$. If $M$ is such a Markov chain
then its associated {\em transition matrix} $Q$ is defined as
$Q:=(q_{i,j})_{i,j \in {\cal N}}$. The matrix $Q$ is
non-negative and its row sums are all unity. It is infinite dimensional
when the number of states is infinite.
\end{definition}
In the sequel we simply speak of ``Markov chains'' and
assume they satisfy Definition~\ref{def.m}.
\begin{definition}\label{def.s}
\rm
A Markov chain $M$ is {\em ergodic} if it has a stationary distribution 
$\pi=(\pi_x)_{x \in S}$ satisfying
$\pi Q= \pi$ and for every distribution $\sigma \neq \pi$ holds
$\sigma Q \neq \sigma$. 
This stationary distribution $\pi$ satisfies $\pi_x >0$
for all $x \in S$ and $\sum_{x \in S} \pi_x=1$. With
$X_t$ being the state of the Markov chain at epoch $t$
starting from $X_0=x_0 \in S$ we have
\begin{equation}\label{eq.erg}
\lim_{n \rightarrow \infty} \frac{1}{n} \sum_{t=1}^n X_t
=  {\bf E}_{\pi}[X]= \sum_{x \in S} \pi_x x,
\end{equation}
approximating theoretical means by sample means.
An ergodic Markov chain is {\em computable} if its transition probabilities
and stationary distribution are computable.
\end{definition}
\begin{corollary}\label{cor.2}
\rm
{\sc Identification Computable Ergodic Markov Chains}
Consider a c.e. or co-c.e. set of halting algorithms for computable
ergodic  Markov chains. 
Let $M$ be an element of this set.
There is an algorithm identifying $M$ in the limit almost surely from
an infinite sequence $x_1,x_2, \ldots$ of states of $M$ produced by
a run of $M$. The code for $M$ via an appropriate Turing machine
is finite. The learning process uses only a finite initial segment
of the data sequence and takes finite time.
\end{corollary}
\begin{example}\label{exam.e}
\rm
Let $M$ be an ergodic Markov chain with a finite set $S$ of states.
There exists a unique distribution $\pi$
over $S$ with strictly positive probabilities such that
\[ \lim_{s \rightarrow \infty} q_{i,j}^s = \pi_j , \]
for all states $i$ and $j$. In this case we have that
$\pi^0 Q^t \rightarrow \pi$ pointwise
as $t \rightarrow \infty$ and the limit is independent of $\pi^0$.
The stationary distribution $\pi$ is the unique vector
satisfying $\pi Q = \pi$, where $\sum_i \pi_i =1$.
(Necessary and sufficient conditions for ergodicity are that
the chain should be {\em irreducible}, that is
for each pair of states $i,j$
there is an $s \in {\cal N}$ such that $q_{i,j}^s > 0$ (state $j$
can be reached from state $i$ in a finite number of steps);
and {\em aperiodic}, the $\mbox{gcd}\{s:q_{i,j}^s>0\}=1$ for all
$i,j \in T$.

Equation $\pi Q = \pi$ is a system of $N$ linear equations in $N$ unknowns
(the entries $\pi_j$). We can solve the unknowns by elimination of variables:
in the first equation express one variable in terms of the others;
substitute the expression into the remaining equations; repeat this process
until the last equation; solve it and then back substitute until
the total solution is found.  

Since $\pi$ is unique the system of linear
equations has a unique solution. If
the original entries of $Q$ are computable, then this process 
keeps the entries of $\pi$ computable as well. Therefore, if the
transition probabilities of the Markov chain are computable, then the 
stationary distribution $\pi$ is a computable probability mass function.
We now invoke the Ergodic Theorem
approximating theoretical means by sample means \cite{Fe68,La03} as
in \eqref{eq.erg}. 
\end{example}

\section{Computable Measures}\label{sect.3}
In the i.i.d. case we dealt with a process where the
future was independent of the present or the past,
in the Markov case we extended this independence such that the immediate
future is determined by the present but not by the past of too long ago. 
What can be shown if we drop the assumption of independence altogether?
Then we go to measures as defined in Appendix~\ref{sect.measure}.
As far as the authors are aware, for general measures 
there exist neither an approximation as in 
Section~\ref{sect.1} nor an analog of the strong law of large numbers.
However, there is a notion of typicality of an infinite data sequence
for a computable measure in the Martin-L\"of theory of sequential
tests \cite{Ma66} based on Kolmogorov complexity,
and this is what we use.

Let $L \subseteq {\cal N}$ and
$\mu$ be a measure on $L^{\infty}$ in a c.e. or co-c.e. 
set of halting algorithms for computable measures. 
In this paper instead of the common notation $\mu(\Gamma_x)$ we use
the simpler notation $\mu(x)$.
We are given a sequence in $L^{\infty}$
which is typical (Definition~\ref{def.typical} 
in Appendix~\ref{sect.measure}) for $\mu$.
The constituent elements of the sequence are possibly dependent.
The set of typical 
infinite sequences of a computable measure $\mu$ have $\mu$-measure one,
and each typical sequence passes all computable tests for 
$\mu$-randomness in the sense of Martin-L\"of.
This probability model is much more general than i.i.d. drawing 
according to a probability mass function. It includes stationary processes,
ergodic processes, Markov processes of any order, and many other models. In particular, this probability model includes many of the models used in mathematical psychological and cognitive science. 

\begin{theorem}\label{theo.3}
{\sc Computable Measure Identification}
Let $L$ be a set of natural numbers. We are given
an infinite sequence of elements from $L$ and this sequence
is guarantied to be typical for at least one measure in a c.e. or co-c.e.
set of halting algorithms for computable measures.
There is an algorithm which
identifies in the limit (certainly) a computable measure in the c.e. or co-c.e set 
above for which the sequence is typical.
The code for this measure is an appropriate Turing machine
and finite. The learning process takes finite time 
and uses only a finite initial segment
of the data sequence. 
\end{theorem}
The proof is deferred 
to Appendix~\ref{sect.proofs}. 
\footnote{
Theorem~\ref{theo.3} and Theorem~\ref{theo.1} are incomparable although
it is tempting to think the latter is a corollary of the former. 
The infinite sequences considered in Theorem~\ref{theo.3} are typical for
some computable measure. Restricted to i.i.d. measures 
(the case of Theorem~\ref{theo.1}) such sequences are a
proper subset from those resulting from i.i.d. draws from the 
corresponding probability mass function. This is the reason
why the result of Theorem~\ref{theo.3} is ``certain'' and the result
from Theorem~\ref{theo.1} is ``almost surely.''
}
We give an outline of the proof of Theorem~\ref{theo.3}. 
Let ${\cal B}$ be a list of a c.e. or co-c.e. set of 
halting algorithms for computable measures. 
Assume that each measure occurs infinitely many times in ${\cal B}$.
For a measure $\mu$ in the list ${\cal B}$ define
\[
\sigma(j)= \log 1/ \mu (x_1 \ldots x_{j} ) - K(x_1 \ldots x_{j}).
\]
By \eqref{eq.A3} in Appendix~\ref{sect.measure}, 
data sequence $x_1, x_2, \ldots$ 
is typical for $\mu$ iff
$\sup_j \sigma(j) = \sigma < \infty$.
By assumption there exists a measure in ${\cal B}$ 
for which the data sequence is typical. Let $\mu_h$ be such a measure.
Since halting algorithms for $\mu_h$ occur infinitely 
often in the list ${\cal B}$
there is a halting algorithm $\mu_{h'}$ in the list ${\cal B}$
with $\sigma_{h'}=\sigma_h$ and
$\sigma_h < h'$. This means that there exists a measure
$\mu_k$ in ${\cal B}$ for which the data sequence $x_1, x_2, \ldots$ 
is typical and $\sigma_k < k$ with $k$ least. 
\begin{example}
\rm
Let us look at some applications. 
Define the primitive recursive measures as the set of objects 
for which it is decidable that they are measures constructed from
primitive recursive functions. Since this set is computable it is c.e..
The theorem shows that
identification in the limit is possible for primitive recursive measures.

Define the time-bounded measures for any fixed computable time bound as
the set of objects for which it is decidable that they are measures
satisfying this time bound. Since this set is computable it is c.e..
Again, the theorem shows that
identification in the limit is possible for elements from this set. 

Let $L$ be a finite set of cardinality $l$, and $f_1,f_2, \ldots$
be a c.e. set of the primitive recursive functions with domain $L$. 
Computably enumerate the strings $x \in L^*$
lexicographical length-increasing. Then every string can be viewed
as the integer giving its position in this order. Let $\epsilon$ denote the
{\em empty word}, that is, the string of length 0. 
Confusion with the notation $\epsilon$ equals
a small quantity is avoided by the context.
Define $\mu_i(\epsilon) = f_i(\epsilon)/f_i(\epsilon)=1$, and
inductively for $x \in L^*$ and $a \in L$ 
define $\mu_i(xa)= f_i(xa)/ \sum_{a \in L} f_i(xa)$. 
Then $\mu_i(x)=\sum_{a \in L} \mu_i(xa)$ for all $x \in L^*$.
Therefore $\mu_i$ is a measure.
Call the c.e. set $\mu_1, \mu_2, \ldots$ the simple measures.
The theorem shows that
identification in the limit is possible for the set of simple measures.
\end{example}

\section{Prediction}\label{sect.4}
In Section~\ref{sect.1} 
the data are drawn i.i.d. according to an appropriate probability
mass function $p$ on the elements of $L$. Given $p$, 
we can predict the probability $p(a|x_1 ,\ldots, x_n)$ that
the next draw results in an element $a$ when the previous draws
resulted in $x_1 , \ldots , x_n$. (The resulting measure on 
$L^{\infty}$ is called an i.i.d. measure.) Once we have identified $p$,
prediction is possible (actually after a finite but unknown running
time of the identifying algorithm).
The same holds for a ergodic Markov chains (Corollary~\ref{cor.2}). This is reassuring for cognitive scientists and neuroscientists who see prediction as fundamental to cognition (\cite{Cl15}, \cite{El90}, \cite{HS98}, \cite{KFF07}).

For general measures as in Section~\ref{sect.3},
allowing dependent data, the situation is quite different.
We can meet the so-called black swan phenomenon of \cite{Po59}.
Let us give a simple example. The data sequence is $a,a, \dots $
is typical (Definition~\ref{def.typical})
for the measure $\mu_1$ defined by
$\mu_1(x)=1$ for every data sequence $x$
consisting of a finite or infinite string of $a$'s
and $\mu_1(x)=0$ otherwise.
But $a,a, \ldots$ is also typical for the measure $\mu_0$ defined by  
$\mu_0(x) = \frac{1}{2}$ for every string $x$ consisting of 
a finite or infinite string of $a$'s, and $\mu_0(x) = \frac{1}{2}$ for
a string $x$ consisting of initially a fixed number $n$ of 
$a$'s followed by a finite or
infinite string of $b$'s, and 0 otherwise. Then, $\mu_1$ and $\mu_0$
give different predictions with an initial $n$-length sequence of $a$'s.
But given a data sequence consisting initially
of only $a$'s, a sensible algorithm will predict $a$ as the most
likely next symbol.
However, if the initial data sequence consists of $n$ symbols $a$, then
for $\mu_1$ the next symbol will be $a$ with probability 1, 
and for $\mu_0$ the next
symbol is $a$ with probability $\frac{1}{2}$ and $b$ with probability 
$\frac{1}{2}$. Therefore, while the i.i.d. case allows
us to predict reliably, in the dependent case there is in general
no reliable predictor for the next symbol. 
In \cite{BD62}, however, Blackwell and Dubin show that
under certain conditions predictions of two 
measures merge asymptotically almost surely.

\section{Conclusion}\label{sect.concl}
Many psychological theories see learning from data, whether sensory or linguistic, as a central function of the brain. Such learning faces great practical difficulties---the space of possible structures is very large and difficult to search, and the computational power of the brain is limited, and the amount of available data may be limited. But it is not clear under what circumstances such learning is possible even with unlimited data and computational resources. Here we have shown that, under surprisingly general conditions, some positive results about identification in the limit in such contexts can be established. 

Using an infinite sequence of elements (or a finite sequence of 
large enough but unknown length) from a set of natural numbers,
algorithms are exhibited that identify in the limit 
the probability distribution associated with this set. 
This happens in two cases.
(i) The underlying set is countable and the target distribution is 
a probability mass function (i.i.d. measure) in a c.e. or co-c.e. set
of computable probability mass functions. 
The elements of the sequence are drawn i.i.d. according
to this probability (Theorem~\ref{theo.1}). This result is extended to
computable ergodic Markov chains (Corollary~\ref{cor.2}).
(ii) The underlying set is countable and the infinite sequence 
is possibly dependent and is typical for a computable measure in a 
c.e. or co-c.e. set of computable measures (Theorem~\ref{theo.3}).

In the i.i.d. case and the ergodic Markov chain case the target 
is identified in the limit {\em almost surely}, 
and in the dependent case the target computable measure is
identified in the limit {\em surely}---however it is not unique but
one out of a set of satisfactory computable measures. 
In the i.i.d. case and Markov case we use the strong law of large numbers. 
For the dependent case we 
use typicality according to the theory 
developed by Martin-L\"of in \cite{Ma66} embedded in the theory of
Kolmogorov complexity. 

In both the i.i.d., the Markovian, and the dependent settings, 
eventually we guess
an index of the target (or one target out of some possible targets in the 
measure case) and stick to this guess forever. This last guess is correct. 
However, we do not know when the guess becomes permanent.
We use only a finite unknown-length
initial segment of the data sequence. 
The target for which the guess is correct
is described by a an appropriate Turing machine computing
the probability mass function, Markov chain, or measure, respectively.

These results concerning algorithms for identification in the limit consider what one might term the ``outer limits'' of what is learnable, by abstracting away from computational restrictions and a finite amount of data available to human learners. Nonetheless, such general results may be informative when attempting to understand what is learnable in more restricted settings. Most straightforwardly, that which is not learnable in the unrestricted case will, {\em a fortiori}, not be learnable when computational or data restrictions are added. It is also possible that some of the proof techniques used in the present context can be adapted to analyse more restricted, and hence more cognitively realistic, settings. 

\appendix

\subsection{Computability}\label{sect.comput}
We can interpret a pair of integers such as $(a,b)$ as rational $a/b$.
A real function $f$ with rational argument is \emph{lower semicomputable}
if it is defined by a rational-valued computable function $\phi(x,k)$
 with $x$ a rational number and $k$ a nonnegative integer
such that $\phi(x,k+1) \geq \phi(x,k)$ for every $k$ and
  $\lim_{k \rightarrow \infty} \phi (x,k)=f(x)$.
This means that $f$ 
 can be computably approximated arbitrary close from below
 (see \cite{LV08}, p. 35).  A function $f$ is  \emph{upper semicomputable}
if $-f$ is semicomputable from below.
 If a real function is both lower semicomputable 
and upper semicomputable then it is \emph{computable}.
A function $f: {\cal N} \rightarrow {\cal R}^+$ 
is a {\em probability mass function} if 
$\sum_x f(x) = 1$.
It is customary to write $p(x)$ for
$f(x)$ if the function involved is a probability mass function.

A set $A \subseteq {\cal N}$ is {\em computable enumerable} (c.e.)  when
we can compute the enumeration $a_1,a_2, \ldots $ with $a_i \in A$ ($i \geq 1$).
A c.e. set is also
called recursively enumerable (r.e.). A {\em co-c.e.} set 
$B \subseteq {\cal N}$ is a set whose 
complement ${\cal N} \setminus B$ is c.e.. 
(A set is c.e. iff it is at level $\Sigma_1^0$ of the arithmetic hierarchy
and it is co-c.e. iff it is at level $\Pi_1^0$.)
If a set is both c.e. and co-c.e. then it is {\em computable}. 
A {\em halting} algorithm for a {\em computable function} 
$f: {\cal N} \rightarrow {\cal R}$ is an algorithm which given  
an argument $x$ and any rational $\epsilon > 0$ computes a total
computable rational function
$\hat {f}: {\cal N} \times {\cal Q} \rightarrow {\cal Q}$ 
such that $|f(x)-\hat {f}(x, \epsilon)| \leq \epsilon$.

\begin{example}
\rm
We give an example of the relation between co-c.e. and identification in the
limit. Consider a  c.e. set $A$ 
of objects and the co-c.e. set $B$ 
such that ${\cal N} \setminus B=A$.
We call the members of $B$ 
the good objects and the members of $A$ the bad objects. 
We do not know in what order the bad objects are enumerated or repeated; 
however we do know that the remaining items are the good objects.
These good objects with possible repetitions 
form the enumeration $B$. 
It takes unknown time to 
enumerate an initial segment of $B$,
but we are sure this happens eventually. 
Hence to identify the $k$th element in the enumeration $B$
requires identification of the first $1, \ldots, k-1$ 
elements. This constitutes 
identification in the limit.
\end{example}

\begin{example}
\rm
It is known that the overwhelming majority of real numbers are not
computable. If a real number
$a$ is lower semicomputable but not computable, 
then we can computably find nonnegative
integers $a_1, a_2, \ldots$ and $b_1, b_2, \ldots$ such that
$a_n/b_n \leq a_{n+1}/b_{n+1}$ and
$\lim_{n \rightarrow \infty} a_n/b_n = a$. If $a$ is the
probability of success in a trial then this gives an example
of a lower semicomputable probability mass function which is not computable.
\end{example}
Suppose we are concerned with all and only computable probability
mass functions. There are countably many since there are only
countably many computable functions. But can we computably enumerate them?
\begin{lemma}\label{lem.incomp}
(i) Let $L \subseteq {\cal N}$ and infinite.
The computable positive probability mass functions 
on $L$ are not c.e.. 

(ii) Let $L \subseteq {\cal N}$ with $|L|\geq 2$. 
The computable positive measures on $L$ are not c.e..
\end{lemma}
\begin{proof}
(i) Assume to the contrary that the lemma is false 
and the computable enumeration is
$p_1,p_2, \ldots .$
Compute a probability mass function $p$
with $p(a) \neq p_i(a_i)$ for $a_i \in L$ is the $i$th element of $L$ 
as follows. If $i$ is odd then 
$p(a_i) := p_i(a_i)+p_i(a_i)p_{i+1}(a_{i+1})$
and $p(a_{i+1}) := p_{i+1}(a_{i+1})-p_i(a_i)p_{i+1}(a_{i+1})$.
By construction $p$ is a computable positive probability
mass function but
different from any $p_i$ in the enumeration $p_1, p_2, \ldots$.

(ii) The set $L^*$ is c.e..
Hence the set of cylinders in $L^{\infty}$ is c.e..
Therefore (ii) reduces to (i). 
\end{proof}
\begin{remark}
\rm
Every probability mass function is positive on some support 
$L \neq \emptyset$ and 0 otherwise. Hence Lemma~\ref{lem.incomp} holds for
all probability mass functions.
\end{remark}

\subsection{Kolmogorov Complexity}\label{sect.kolmcomp}
We need the theory of Kolmogorov complexity \cite{LV08} (originally
in \cite{Ko65} and the prefix version we use here originally in \cite{Le74}). 
A prefix Turing machine is a Turing machine with a one-way 
read-only input
tape with a distinguished tape cell called the {\em origin}, 
a finite number of two-way read-write working tapes 
on which the computation takes place, an auxiliary tape on
which the auxiliary string $y \in \{0,1\}^*$ is written, and
a one-way write-only output tape. 
At the start of the computation the input tape
is infinitely inscribed from the origin onwards, and the input head
is on the origin. The machine operates with a binary alphabet.
If the machine halts then the input head has scanned
a segment of the input tape from the origin onwards. We
call this initial segment the {\em program}.  

By the construction above, for every auxiliary $y \in \{0,1\}^*$, 
the set of programs is a prefix code: no program is a proper
prefix of any other program. 
Consider a standard enumeration of all prefix Turing
machines 
\[
T_1, T_2, \ldots . 
\]
Let $U$ denote a 
prefix Turing machine such that for every $z,y\in\{0,1\}^{*}$ and $i\geq1$
we have $U(i,z,y)=T_{i}(z,y)$. That is, for each finite binary program
$z$, auxiliary $y$, and machine index $i\geq1$, 
we have that $U$'s execution
on inputs $i$ and $z,y$ results in the same output as that obtained
by executing $T_{i}$ on input $z,y$. We call such a $U$ a {\em universal}
prefix Turing machine. 

However, there are more ways a prefix Turing machine
can simulate other prefix Turing machines. For example,
let $U'$ be such that $U'(i,zz,y)=T_i(z,y)$ for all $i$ and $z,y$,
and $U'(p)=0$ for $p \neq i,zz,y$ for some $i,z,y$.
Then $U'$ is universal also. To distinguish machines like $U$ 
with nonredundant input from
other universal machines, Kolmogorov \cite{Ko65} called them {\em optimal}. 

Fix an optimal machine, say $U$. Define
the conditional {\em prefix Kolmogorov complexity} $K(x|y)$ 
for all $x,y \in \{0,1\}^*$ by 
$K(x|y)=\min_p \{ |p|: p\in\{0,1\}^{*}\:\textrm{and}\: U(p,y)=x\}$.
(Here $U$ has two arguments rather than three. We consider the first argument
to encode the first two arguments of the previous three.)
For the same $U$, define the {\em time-bounded conditional 
prefix Kolmogorov complexity}
$K^t(x|y)=\min_p\{ |p|:\, p\in\{0,1\}^{*}\:
\textrm{and}\; U(p,y)=x\;\textrm{in $t$ steps}\} $.
To obtain the unconditional versions of the prefix Kolmogorov complexities
set $y = \epsilon$ where $\epsilon$ is the {\em empty word}
(the word with no letters). 
It can be shown that $K(x|y)$ is incomputable \cite{Ko65}. 
Clearly $K^t(x|y)$
is computable if $t < \infty$. Moreover,
$K^{t'} (x|y) \leq K^t(x|y)$ for every $t' \geq t$, and
$\lim_{t \rightarrow \infty} K^t(x|y) = K(x|y)$.

\subsection{Measures and Computability}\label{sect.measure}
Let $L \subseteq {\cal N}$.
Given a finite sequence $x=x_1, x_2, \ldots , x_n$ of elements
of $L$, we consider the set of
infinite sequences starting with $x$. The set of all such sequences is
written as $\Gamma_x$, the {\em cylinder} of $x$. 
We associate a
probability $\mu (\Gamma_x)$ with the event 
that an element of $\Gamma_x$ occurs.
Here we simplify the notation $\mu (\Gamma_x)$ and write $\mu (x)$.
The transitive closure of the intersection, complement, 
and countable union of cylinders gives a set of subsets of $L^{\infty}$.
The probabilities associated with these subsets are derived from the
probabilities of the cylinders in standard ways \cite{Ko33a}.
A {\em measure} $\mu$ satisfies the following equalities:
\begin{eqnarray}\label{eq.me}
&& \mu (\epsilon) = 1 
\\&& \mu (x) = \sum_{a \in L} \mu(xa). 
\nonumber
\end{eqnarray}
Let $x_1 ,x_2 , \ldots$ be an infinite sequence of elements of
$L$. The sequence is typical
for a computable measure $\mu$ if it passes 
all computable sequential tests (known and unknown alike)
for randomness with respect to $\mu$. These tests are formalized by
Martin-L\"of \cite{Ma66}. One of the highlights of the theory 
of Martin-L\"of is that the sequence passes all these tests iff it passes
a single computable universal test, \cite[Corollary 4.5.2 on p 315]{LV08},
see also \cite{Ma66}.
\begin{definition}\label{def.typical}
\rm
Let $x_1 ,x_2 , \ldots$ be an infinite sequence of elements of
$L \subseteq {\cal N}$. The sequence is {\em typical} or
{\em random} for a computable measure $\mu$ iff 
\begin{equation}\label{eq.A3}
\sup_n \{\log \frac{1}{\mu (x_1 \ldots x_n )} - K(x_1 \ldots x_n )\}
< \infty .  
\end{equation}
\end{definition}
The set of infinite sequences that are typical with respect
to a measure $\mu$ have $\mu$-measure one. 
The theory and properties of such sequences for computable measures
 are extensively treated 
in \cite[Chapter 4]{LV08}. There the term $K(x_1 \ldots x_n )$ in 
\eqref{eq.A3} is given as $K(x_1 \ldots x_n |\mu)$. However, since 
$\mu$ is computable we have $K(\mu) < \infty$ and therefore 
$K(x_1 \ldots x_n |\mu) \leq K(x_1 \ldots x_n)+O(1)$.

\begin{example}
\rm
Let $k$ be a positive integer and fix an $a \in \{1, \ldots, k\}$.
Define  measure $\mu_k$ 
by $\mu_k(\epsilon)=1$ and $\mu_k(x_ 1\ldots x_n)=1/k$ for $n \geq 1$ 
and $x_i = a$ for
every $1 \leq i \leq n$,
and $\mu_k(x_ 1\ldots x_n)= (k^{-n}-1)(1-1/k)$ otherwise. 
Then $K(a \ldots a)$ (a sequence
of $n$ elements $a$) equals 
$K(i,n)+O(1) = O(\log n + \log k)$.
(A sequence of $n$ elements $a$ is described by $n$ in $O(\log n)$
bits and $a$ in $O(\log k)$ bits.) 
By \eqref{eq.A3} we have $\sup_{n \in {\cal N}} \{ \log 1/\mu_k(a \ldots a) 
- K(a \ldots a) \} < \infty$. Therefore the infinite 
sequence $a,a, \ldots$ is typical
for every $\mu_k$. However, 
the infinite sequence $y_1,y_2, \ldots $ 
is not typical for $\mu_k$ with $y_i \in \{1, \ldots ,k\}$ ($i \geq 1$)
and $y_i \neq y_{i+1}$ for some $i$.
Namely, $\sup_{n \in {\cal N}} \{ 1/\mu_k(y_1 y_2\ldots y_n) - 
K(y_1y_2\ldots y_n) \} = \infty$.
\end{example}
Since $k$ can be any positive integer, 
the example shows that an infinite sequence of data can be typical for more
than one measure. Hence our task is not to identify a single 
computable measure according to which the
data sequence was generated as a typical sequence,
but to identify a computable measure that {\em could}
have generated the data sequence as a typical sequence.

\subsection{Proofs of the Theorems}\label{sect.proofs}
\begin{proof} {\sc of Theorem~\ref{theo.1}: I.I.D. Computable Probability 
Identification}. Let $L \subseteq {\cal N}$, and
$X_1, X_2, \ldots $ be a sequence of mutually independent random variables,
each of which is a copy of a single random
variable $X$ with probability mass function $P(X=a)=p(a)$ for $a \in L$. 
Without loss of generality $p(a)>0$ for all $a \in L$.
Let $\#a(x_1,x_2, \ldots, x_n)$ denote the number of times 
$x_i =a$ ($1 \leq i \leq n$) for some fixed $a \in L$.
\begin{claim}\label{claim.slln}
If the outcomes of the random variables $X_1,X_2, \ldots$ are 
$x_1,x_2, \ldots ,$
then almost surely for all $a \in L$ we have
\begin{equation}\label{eq.strong}
\lim_{n \rightarrow \infty} \; 
\left( p(a)-\frac{\#a(x_1,x_2, \ldots ,x_n)}{n} \right) = 0.
\end{equation}
\end{claim}
\begin{proof}
The strong law of large numbers (originally in \cite{Ko33}, see also \cite{Ko33a} and \cite{Ca17})
states that if we perform the same 
experiment a large number of times, then almost surely 
the number of successes divided by the number of trials
goes to the expected value, 
provided the mean exists, see 
the theorem on top of page 260 in \cite{Fe68}. 
To determine the probability of an $a \in L$ we consider the 
random variables $X_a$ with just two outcomes $\{a, \bar{a}\}$. 
This $X_a$ is
a Bernoulli process $(q_a,1-q_a)$ where 
$q_a=p(a)$ is the probability of $a$ and 
$1-q_a= \sum_{b \in L\setminus \{a\}} p(b)$ 
is the probability of $\bar{a}$. If we set
$\bar{a} = \min \; (L \setminus \{a\})$, 
then the mean $\mu_a$ of $X_a$ is 
\[
\mu_a = aq_a+\bar{a}(1-q_a) \leq \max \{a, \bar{a}\} < \infty.
\]
Thus, every $a \in L$ incurs a random
variable $X_a$ with a finite mean.
Therefore, 
$(1/n)\sum_{i=1}^n (X_a)_i$ converges 
almost surely to $q_a$ as $n \rightarrow \infty$.
The claim follows.
\end{proof}

Let ${\cal A}$ be a list of a c.e. or co-c.e. set of halting algorithms for the
computable probability mass functions.
If $q \in {\cal A}$ and $q = p$ then for every $\epsilon >0$ 
and $a \in L$ holds $p(a)-q(a)\leq \epsilon$.
By Claim~\ref{claim.slln}, almost surely  
\begin{equation}\label{eq.=0}
\lim_{n \rightarrow \infty} \max_{a \in L} \; 
\left(q_i(a) -\frac{\#a(x_1,x_2, \ldots ,x_n)}{n} \right)
= 0.
\end{equation}
If $q \in {\cal A}$ and $q \neq p$ then 
there is an $a \in L$ and a constant $\delta >0$ 
such that $|p(a)-q(a)| > \delta$. 
Again by Claim~\ref{claim.slln}, almost surely
\begin{equation}\label{eq.>0}
\lim_{n \rightarrow \infty} \; 
\max_{a \in L} \left|q_i(a) -\frac{\#a(x_1,x_2, \ldots ,x_n)}{n}\right|
> \delta .
\end{equation}
In the proof \cite[p. 204]{Fe68} of the strong law of large numbers
it is shown that if we draw $x_1,x_2, \ldots$ i.i.d. from a set 
$L \subseteq {\cal N}$ according to a probability mass function $p$
then almost surely the size of the fluctuations in going to the limit
\eqref{eq.=0} 
satisfies $|np(a) - \#a(x_1,x_2, \ldots ,x_n)|/\sqrt{np(a)p(\bar{a})}
< \sqrt{2 \lambda \lg n}$ for every $\lambda >1$ and $n$ is
large enough, for all $a \in L$.
Here $\lg$ denotes the natural logarithm. 
Since $p(a)p(\bar{a}) \leq \frac{1}{4}$ and $\lambda = \sqrt{2}$
it suffices that $|p(a) - \#a(x_1,x_2, \ldots ,x_n)/n|
< \sqrt{(\lg n)/n}$ for all but finitely many $n$.

Let $q \in {\cal A}$. 
For $q \neq p$ there is an $a \in L$ such that
by \eqref{eq.>0} and the fluctuations in going to that
limit we have $|q(a) - \#a(x_1,x_2, \ldots ,x_n)/n|
> \delta - \sqrt{(\lg n)/n}$ for all but finitely many $n$. 
Since $\delta >0$ is constant,
we have $2\sqrt{(\lg n)/n} < \delta$ for all but finitely many $n$. Hence
$|q(a) - \#a(x_1,x_2, \ldots ,x_n)/n| > \sqrt{(\lg n)/n}$
for all but finitely many $n$.

Let ${\cal A}=q_1,q_2, \ldots$ and $p=q_k$ with $k$ least.
We give an algorithm  with as output a sequence of indexes 
$i_1,i_2, \ldots$ such that all but finitely
many indexes are $k$.
If $L=\{a_1,a_2, \ldots \}$ is infinite then the algorithm will only use a finite subset of it.
Hence we need to define this finite subset
and show that the remaining elements can be ignored.
Let $A_n = \{a \in L: \#a(x_1,x_2, \ldots ,x_n) > 0\}$. 
In case $a \in L$  but $a \not\in A_n$ we still have 
$|q_k(a) - \#a(x_1,x_2, \ldots ,x_n)/n| \leq 
\sqrt{(\lg n)/n}$ for all but finitely many $n$. 

Now define the following sets. For each $q_i \in {\cal A}$ the set
$B_{i,n} = \{a_1, \ldots, a_m\}$ with $m$ least such that 
$\sum_{j=m+1}^{\infty} q_i(a_j) = 1-\sum_{j=1}^m q_i(a_j) 
< \sqrt{1/n}$.
Therefore, if $a \in L \setminus B_{i,n}$ then $q_i(a) < \sqrt{1/n}$.
In contrast to the infinity of $L$ the sets $A_n$ and 
$B_{i,n}$ are finite for all $n$ and $i$. 

Define $L_{i,n}=A_n \bigcup B_{i,n}$. 
Since $L_{i,n} \subseteq L$ we have for every $a \in L_{i,n}$ that
$|q_k(a)- \#a(x_1,x_2, \ldots ,x_n)/n | \leq  \sqrt{(\lg n)/n}$
for all but finitely many $n$. However, for $q_i \neq q_k$ there is an 
$a \in L_{i,n}$ but no $a \in L \setminus L_{i,n}$ 
such that $|q_i(a)- \#a(x_1,x_2, \ldots ,x_n)/n | >
\sqrt{(\lg n)/n}$ for all but finitely many $n$. This leads to the 
following algorithm with $I$ the set of indexes of the elements in ${\cal A}$:

\begin{tabbing}
{\bf for} \= $n:=1,2, \ldots$\\

\> $I:=\emptyset$; \= {\bf for} \= $i:=1,2, \ldots, n$\\

\> \> {\bf if}
$\max_{a \in L_{i,n}} |q_i(a) - \#a(x_1,x_2, \ldots ,x_n)/n | 
 < \sqrt{(\lg n)/n}$\\

\>\> {\bf then} $I:=I \bigcup \{i\}$; \\

\> \> $i_n := \min I$
\end{tabbing}

With probability 1 
for every $i < k$ for all but finitely many $n$ we have $i \not\in I$ 
while $k \in I$ for all but finitely many $n$. 
(Note that for every $n=1,2, \ldots$ the main term in the above algorithm is
computable even if $L$ is infinite.) 
The theorem is proven. 
\end{proof}

\begin{proof} {\sc of Theorem~\ref{theo.3} Computable Measure Identification}
For the Kolmogorov complexity notions 
see Appendix~\ref{sect.kolmcomp}. For
the theory of computable measures, see 
Appendix~\ref{sect.measure}.
In particular
we use the criterion of Definition~\ref{def.typical} in 
Appendix~\ref{sect.measure} to
show that an infinite sequence is typical 
in Martin-L\"of's sense.  
The given data sequence $x_1,x_2, \ldots$ is by assumption typical 
for some computable measure $\mu$ in a c.e. or co-c.e. set of computable
measures and hence 
satisfies \eqref{eq.A3} with respect to $\mu$.
We stress that the data sequence is possibly 
typical for different computable measures.
Therefore we cannot speak of the single {\em true} computable measure, 
but only of {\em a} computable measure
for which the data is typical. 

Let ${\cal B}$ be an enumeration of halting algorithms 
for a c.e. or co-c.e. set of
computable measures such that each element occurs infinitely many
times in the list. If the enumeration is such that each element occurs
ony finitely many times, then the enumeration can be changed into one
where each element occurs infinitely many times. For instance, by repeating
the first element after every position in the original enumeration, 
repeating the second element in the original enumeration 
after every second position in the resulting enumeration, and so on.

\begin{claim}\label{claim.typical}
\rm
There is an algorithm with as input an enumeration 
${\cal B}= \mu_1,\mu_2, \ldots$
and as output a sequence 
of indexes $i_1,i_2, \ldots$. For every large enough $n$ we have 
$i_n= k$ with $\mu_k$ a computable measure 
for which the data sequence is typical.
\end{claim}
\begin{proof}
Define for $\mu$ in ${\cal B}$ 
\[
\sigma(j)= \log 1/ \mu (x_1 \ldots x_{j} ) - K(x_1 \ldots x_{j}).
\]
Since  $K$ is upper semicomputable and $\mu$ is computable,
the function $\sigma(j)$
is lower semicomputable for each $j$.
Define the $n$th value in the lower semicomputation of 
$\sigma(j)$ as $\sigma^n(j)$. 
By \eqref{eq.A3}, the data sequence $x_1, x_2, \ldots$ 
is typical for $\mu$ if
$\sup_{j\geq 1} \sigma(j) = \sigma < \infty$ 
In this case, since $\mu$ is lower semicomputable,
$\max_{1 \leq j \leq n} \sigma^n(j) \leq \sigma$ for all $n$.
In contrast, the data sequence is not typical for
$\mu$ if $\sigma(n) \rightarrow \infty$ with $n \rightarrow \infty$
implying $\sigma^n(n) \rightarrow \infty$ with $n \rightarrow \infty$.

By assumption there exists a measure in ${\cal B}$
for which the data sequence is typical. Let $\mu_h$ be such a measure
Since halting algorithms for $\mu_h$ occur infinitely often in the 
enumeration ${\cal B}$
there is a halting algorithm $\mu_{h'}$ in the enumeration ${\cal B}$
with $\sigma_{h'}=\sigma_h$ and
$\sigma_h < h'$. Therefore, there exists a measure
$\mu_k$ in ${\cal B}$ for which the data sequence $x_1, x_2, \ldots$
is typical and $\sigma_k < k$ with $k$ least.
The algorithm to determine $k$ 
is as follows. 

\begin{tabbing}
{\bf for} \= $n:=1,2, \ldots$\\ 

\> {\bf if} $i \leq n$ is least such that  
$\max_{1 \leq j \leq n}\sigma_i^n(j) <i$\\

\>  {\bf then} 
output $i_n = i$ {\bf else} output $i_n=1$.
\end{tabbing}

Eventually $\max_{1 \leq j \leq n}\sigma_k^n(j)<k$ for large enough $n$,
and $k$ is the least index of elements in ${\cal B}$ for which this holds.
Hence there exists an $n_0$  such that $i_n=k$ for all $n \geq n_0$.
\end{proof}

For large enough $n$ we have by Claim~\ref{claim.typical} a test
such that we can identify in the limit an index of a measure 
in ${\cal B}$ for which the provided data sequence is typical. 
Hence there is an $n_0$ such that $i_n=k$ for all $n \geq n_0$.
We do not care what $i_1, \ldots , i_{n-1}$ are.
This proves the theorem.
\end{proof}

\subsection{Genesis of the Result}\label{sect.genesis}
At the request of a referee we give a brief account of the
genesis of the result. In version arXiv:1208.5003 we assumed that we were 
dealing with all computable probabilities and the necessary extensions to
measures. The first part of the technical results dealt with
i.i.d drawing and ergodic Markov chains. Here a main ingredient was
to appeal to the known result that computable 
semiprobability mass functions (those
summing to 1 or less than 1) are computably enumerable in a linear list.
By some tricks we sought to computably extract the probabilities proper
from among them and use the Law of Large Numbers.
For the more difficult dependent case we resorted to measures. Here we used
a known result that the computable semimeasures (where the equality signs in 
the measure
conditions are replaced by inequality $\leq$ signs) are computably enumerable
as well in a linear list. Again we sought to computably extract the measures
proper from this list and use a (known) criterion that says that
the measures for which the provided infinite sequence of examples
is random (typical) keeps a certain quantity finite. The proof
in arXiv:1208.5003 entailed to separate the finite sequences of this quantity
from the infinite ones. This took a long time and effort. Subsequently
in \cite{BMS14} it was shown that the approach of arXiv:1208.5003
was in error: they showed by a very technical argument that
identification of computable probabilities
and computable measures by infinite sequences of examples was impossible.
Extensive email contact with one of the authors, Laurent Bienvenu, showed
that the essential point was the extraction of probabilities and
measures from the above computable enumerations of all computable
semiprobabilities and computable semimeasures. It turned out that
we required computable enumerations or co-computable enumerations of
computable probabilities and computable measures at the outset.
This was done in arXiv:1311.7385. That is, the identification does not
hold for all computable probabilities and computable measures as in the 
too large claims of arXiv:1208.5003 but for the subclass of
computable enumerations or co-computable enumerations of them.
Furthermore the very difficult argument separating bounded infinite
sequences from unbounded ones (in the dependent case) was replaced
by a simple one reminiscent of the h-index in citation science.
Namely, a  bounded infinite sequence has a(n unknown) bound. But if
the measures involved are enumerated then eventually the index of
one (there are infinitely many of them) for which the bound is 
relevant will pass this bound.

\section*{Acknowledgement}
We thank Laurent Bienvenu for pointing out an error in the
an earlier version and elucidating comments. Drafts of this paper
proceeded since 2012 in various states of correctness
through arXiv:1208.5003 to arXiv:1311.7385.

\end{document}